\documentclass[10pt,twocolumn,letterpaper]{article}

\usepackage{cvpr}
\usepackage{times}
\usepackage{epsfig}
\usepackage{graphicx}
\usepackage{algorithm,algorithmic}
\usepackage{amsmath}
\usepackage{amssymb}
\usepackage{subfigure}

\newtheorem{theorem}{Theorem}
\newtheorem{lemma}{Lemma}
\newtheorem{proof}{Proof}
\newcommand{\RR}[2]{\mathbf{#1} \in \mathbb{R}^{#2}}

\usepackage[pagebackref=true,breaklinks=true,letterpaper=true,colorlinks,bookmarks=false]{hyperref}

\cvprfinalcopy 

\def\cvprPaperID{5394} 

\ifcvprfinal\fi
\begin{document}

\title{Frequency Pooling: Shift-Equivalent and Anti-Aliasing Downsampling}

\author{Zhendong Zhang\\
School of Electronic Engineering, Xidian University\\
Xi'an, 710071, China\\
{\tt\small zhd.zhang.ai@gmail.com}
}

\maketitle

\begin{abstract}
	Convolution utilizes a shift-equivalent prior of images, thus leading to a great success in image processing tasks. However, commonly used poolings in convolutional neural networks (CNNs), such as max-pooling, average-pooling, and strided-convolution, are not shift-equivalent. Thus, the shift-equivalence of CNNs is destroyed when convolutions and poolings are stacked. Moreover, anti-aliasing is another essential property of poolings from the perspective of signal processing. However, recent poolings are neither shift-equivalent nor anti-aliasing. To address this issue, we propose a new pooling method that is shift-equivalent and anti-aliasing, named frequency pooling. Frequency pooling first transforms the features into the frequency domain, and then removes the frequency components beyond the Nyquist frequency. Finally, it transforms the features back to the spatial domain. We prove that frequency pooling is shift-equivalent and anti-aliasing based on the property of Fourier transform and Nyquist frequency.
	Experiments on image classification show that frequency pooling improves accuracy and robustness with respect to the shifts of CNNs.
\end{abstract}

\section{Introduction}

Convolutional neural networks (CNNs) have achieved great success in image processing \cite{goodfellow2016deep}, natural language processing \cite{yin2017comparative}, and game playing \cite{Mnih2013Playing}. One of the reasons is that convolutions utilize the shift-equivalent prior of signals. Modern CNNs include not only convolutional layers but also downsampling or pooling layers. As an important part of CNNs, poolings are used to reduce spatial resolution and aggregate spatial information. Poolings reduce the memory and computational costs of CNNs.

From the signal processing viewpoint, pooling, as a special case of downsampling, should be anti-aliasing. Otherwise, pooling causes frequency aliasing, i.e. the high frequency components of the original signal are projected into the low frequency components of the downsampled signal by pooling, and thus pooling mixes the high and low frequency components. This leads to sub-optimal reconstruction results while misleading the subsequent processing. For anti-alias sampling, the sampling rate should be at least as twice as the highest frequency based on classical Nyquist sampling theorem \cite{nyquist1928certain}. A traditional solution of anti-aliasing is applying a low-pass filter to a signal before downsampling. Inspired by it, early CNNs use average pooling to achieve downsampling \cite{lecun1989handwritten}. Later, empirical evidence suggests max-pooling \cite{scherer2010evaluation} and strided-convolution \cite{long2015fully} provide better performance in accuracy. However, average pooling performs better in anti-aliasing than max-pooling and strided-convolutions \cite{zhang2019making}. 

Shift-equivalence is another essential property of pooling. When convolutions and poolings of no shift-equivalence are stacked, they destroy shift-equivalence of CNNs. Unfortunately, commonly used poolings are not shift-equivalent. To make matters worse, small shifts in the input can drastically change the output when stacking multiple max-poolings or strided-convolutions \cite{zhang2019making, engstrom2017rotation, azulay2018deep}. We provide some examples on ImageNet in Fig. \ref{fig:probabilty}. Shift-equivalence is expected to be a fundamental property of CNNs. However, the fact that CNNs with poolings are not shift-equivalent has been ignored by the machine learning community. Until now, this phenomenon is pointed out by Zhang \cite{zhang2019making}. To make CNNs shift-equivalent, he proposed an anti-aliasing pooling (AA-pooling) by applying a low-pass filter to the signal before downsampling. He claimed that anti-aliasing improves the shift invariance of poolings. He observed better accuracy and generalization in image classification when low-pass filtering is integrated correctly. Although AA-pooling reduces the aliasing effects and makes the output more stable with respect to the input shifts, it is neither shift-equivalent nor anti-aliasing in theory.

In this paper, we propose a shift-equivalent and anti-aliasing pooling. We first transform a signal or an image into the frequency domain via Discrete Fourier Transform (DFT). Then, we only retain its low frequency components, i.e. the frequency components that are smaller than the Nyquist sampling rate. Finally, we transform the low frequency components back into the time or spatial domain via inverse DFT (IDFT). We call this pooling as frequency pooling (F-pooling). Note that a similar method is proposed by Ryu et al. \cite{ryu2018dft-based}. However, they only focused on the classification accuracy without considering the shift-equivalence of pooling in both theory and practice. 

Compared with existing methods, the novelties and contributions of F-pooling are summarized as follows:

\begin{itemize}
	\item To our knowledge, there is no formal definition of shift-equivalence when we apply poolings to discrete signals. In this work, we formally define the shift-equivalence. A suitable upsampling $\mathcal{U}$ needs to be involved in the definition. Without $\mathcal{U}$, the definition for discrete signals is ill-posed. We believe that this mathematical formulation has a great impact on further research.
	
	\item We prove that F-pooling is an optimal anti-aliasing downsampling from the perspective of reconstruction. We also prove that F-pooling is shift-equivalent. The upsampling $\mathcal{U}$ plays an important role in the proof of shift-equivalence. To best of our knowledge, F-pooling is the first pooling method that has both properties. We further investigate whether the shift-equivalence of F-pooling is transitive. Roughly, the composite of convolutions and an F-pooling is still shift-equivalent while the composite of multiple F-poolings may be not shift-equivalent.
	
	\item Experiments on CIFAR-100 and a subset of ImageNet demonstrate that F-pooling remarkably increases accuracy and robustness with respect to the shifts of commonly used network architectures. Moreover, the shift consistency of F-pooling is improved more when we replace zero padding of convolutions with circular padding.
	
\end{itemize}

\begin{figure*}
	\centering
	\subfigure[Max-pooling]{
		\includegraphics[width=0.31\textwidth]{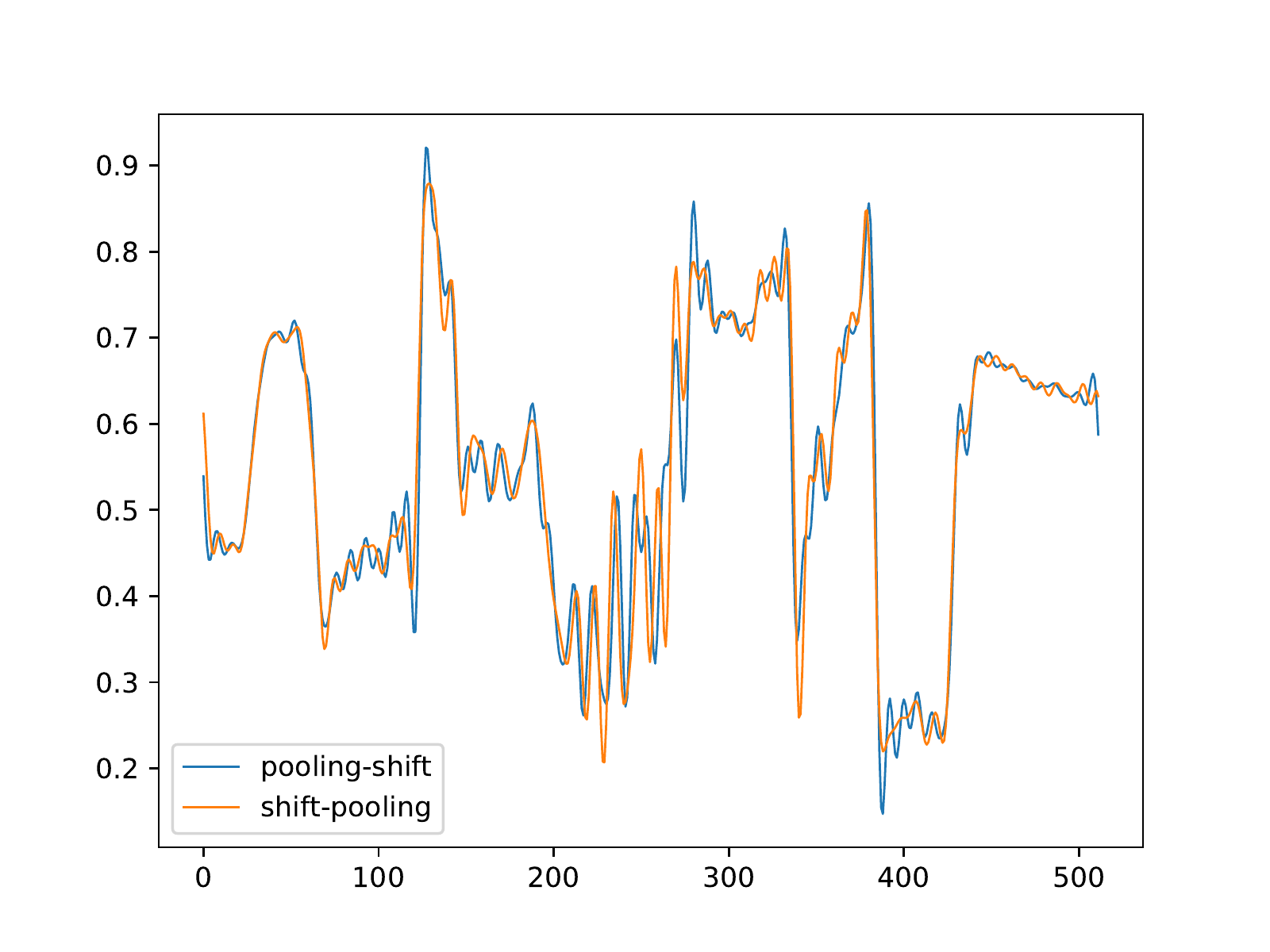}
	}
	\subfigure[Avg-pooling]{
		\includegraphics[width=0.31\textwidth]{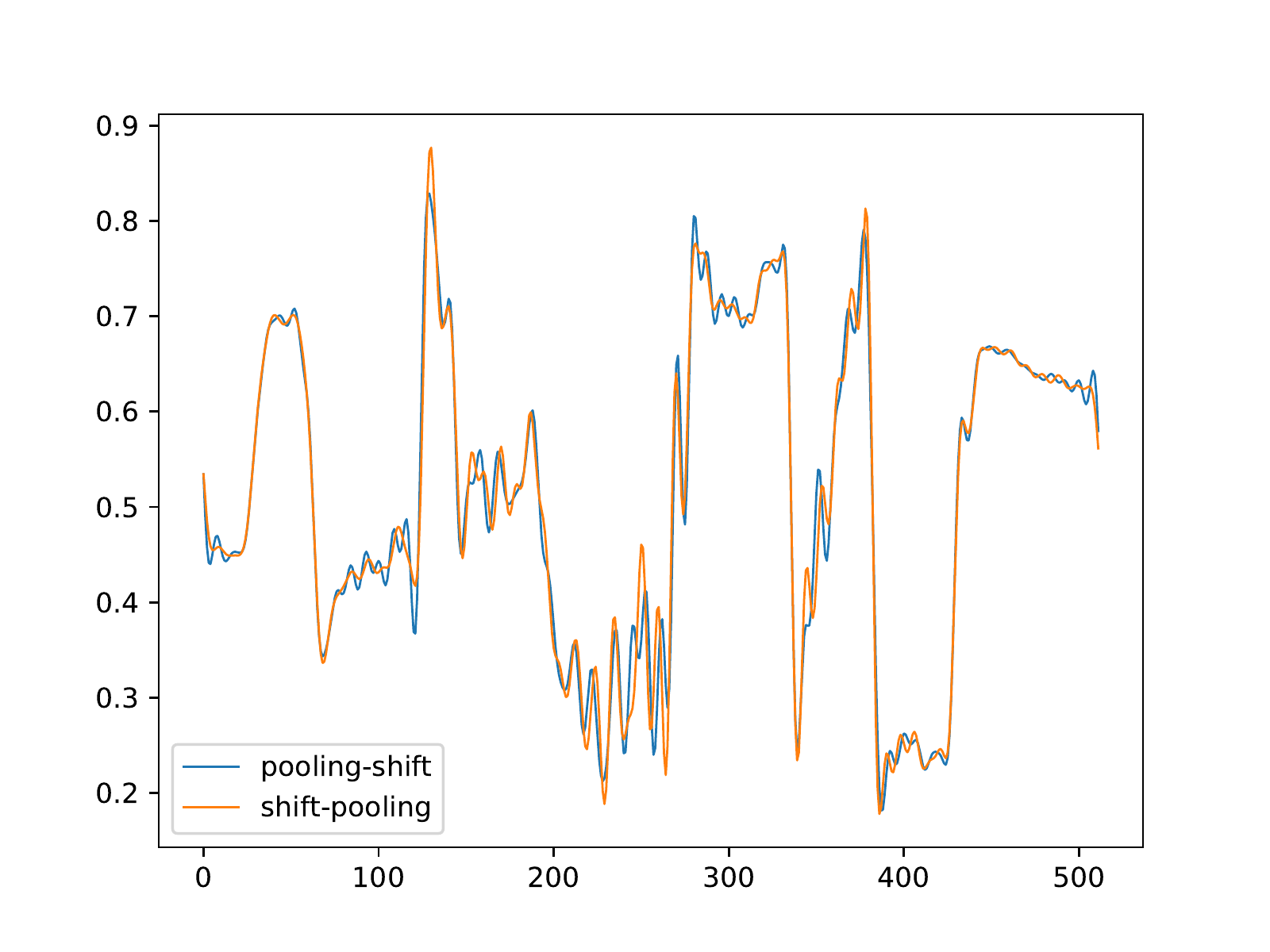}
	}
	\subfigure[F-pooling]{
		\includegraphics[width=0.31\textwidth]{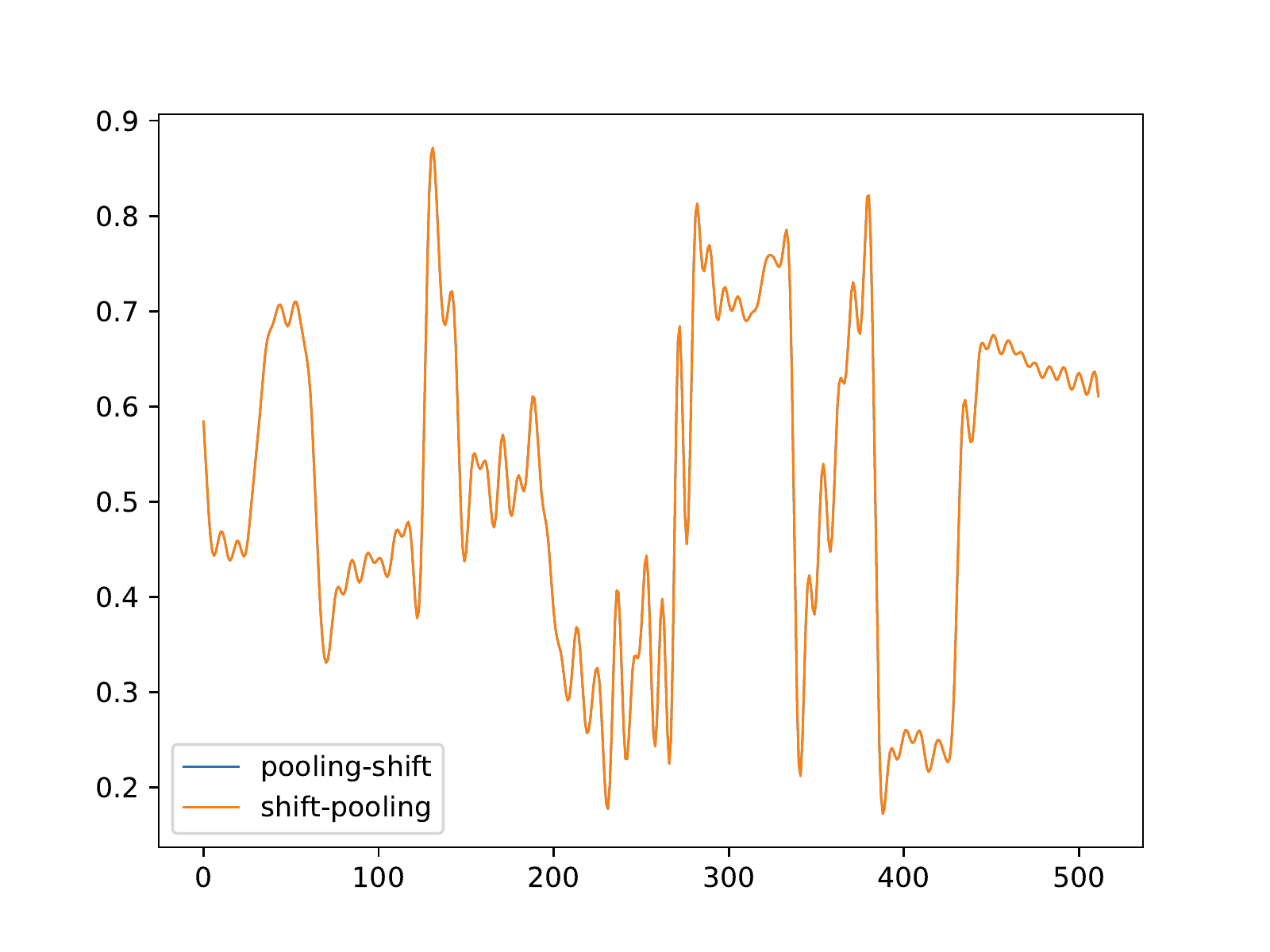}
	}
	\caption{Tests of shift-equivalence on 1D signals. F-pooling is frequency pooling. Blue lines are obtained by pooling, upsampling, and shift in order. Orange lines are obtained by shift, pooling, and upsampling in order. All poolings downsample signals by factor 4. The shift operation shifts signals by 2 pixels. The upsampling operation is set to Eq. \eqref{eq:up}. Best viewed on screen.}
	\label{fig:test}
\end{figure*}

\section{Related Work}
Pooling which reduces the resolution of feature maps is an important function of CNNs. Early CNNs \cite{lecun1989handwritten} use average pooling that is good for anti-aliasing. Later empirical evidence reveals that max-pooling \cite{scherer2010evaluation} and strided-convolutions \cite{long2015fully} provide higher accuracy. However, small shifts in the input can drastically change the output when stacking multiple max-pooling or strided-convolutions \cite{engstrom2017rotation, azulay2018deep}. Other variants such as \cite{graham2014fractional, he2015spatial, lee2016generalizing} (we list a few of them), focus on extending the functionality of pooling \cite{lee2016generalizing} or making pooling adjusted to variable input/output size \cite{graham2014fractional, he2015spatial}. For F-pooling, its input/output size is variable. However, it is not the focus of this work.

Recently, Zhang \cite{zhang2019making} shows that CNNs are more robust to image shifts when low-pass filtering is integrated correctly. Specifically, he decomposes a pooling with downsampling factor $s$ into two parts: a box filter with factor $s$ and a pooling with factor $1$. In this way, he claims that the pooling becomes anti-aliasing due to the smoothness of box filter. Moreover, he shows that anti-aliasing improves the shift robustness of CNNs. However, his method is not strictly shift-equivalent and anti-aliasing in theory. Ryu et al. \cite{ryu2018dft-based} proposes a similar method. They also downsample the features in frequency domain and they show improved accuracy of their method. Williams and Li \cite{williams2018wavelet} propose Wavelet-pooling. They decompose a signal via wavelet transform to retain the lowest sub-band. This process is repeated until the designed down sampling factor is met. They claim that Wavelet-pooling is better than others because it is a global transform instead of a local transform. Two works \cite{ryu2018dft-based,williams2018wavelet} have the same idea as F-pooling in view of removing high frequency components. However, they don't realize that their methods are more robust to the shift. Thus, they neither test the shift-equivalence experimentally nor prove the intrinsic property of their methods theoretically. Mallat et al. \cite{Mallat2013Invariant} and Sirfre and Mallat \cite{2013Rotation} propose an invariant image representation to basic transformations including shifts. However, the shift invariance holds only for continuous signals and the downsamplings are not analyzed.

F-pooling is a complex transform in which DFT and IDFT are involved. Many works integrate complex transform or complex values into neural networks. Amin and Murase \cite{amin2009single} study single-layered complex-valued neural networks for real-valued classification problems. Complex numbers represented in polar coordinates are more suitable for handling rotations. Based on this, Cohen et al. \cite{s.2018spherical} propose spherical CNNs which are rotation-equivalent to deal with signals projected from a spherical surface. Trabelsi et al. \cite{trabelsi2018deep} propose general deep complex CNNs. They adjust batch normalization and non-linear activation for complex CNNs. F-pooling can be used in their method without considering the imaginary part. F-pooling utilizes the shift theorem of DFT and achieves shift-equivalence. This is a new success for the combination of complex transform and neural networks.

\begin{figure*}
	\centering
	
	\begin{subfigure}
		\centering
		{\includegraphics[width=0.25\linewidth]{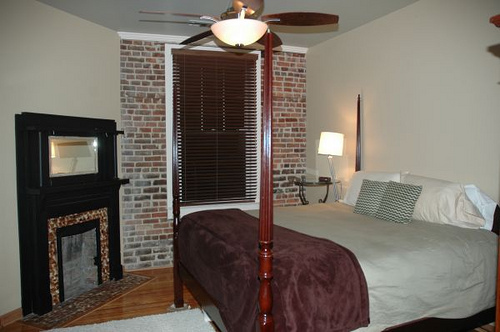}}
		{\includegraphics[width=0.65\linewidth]{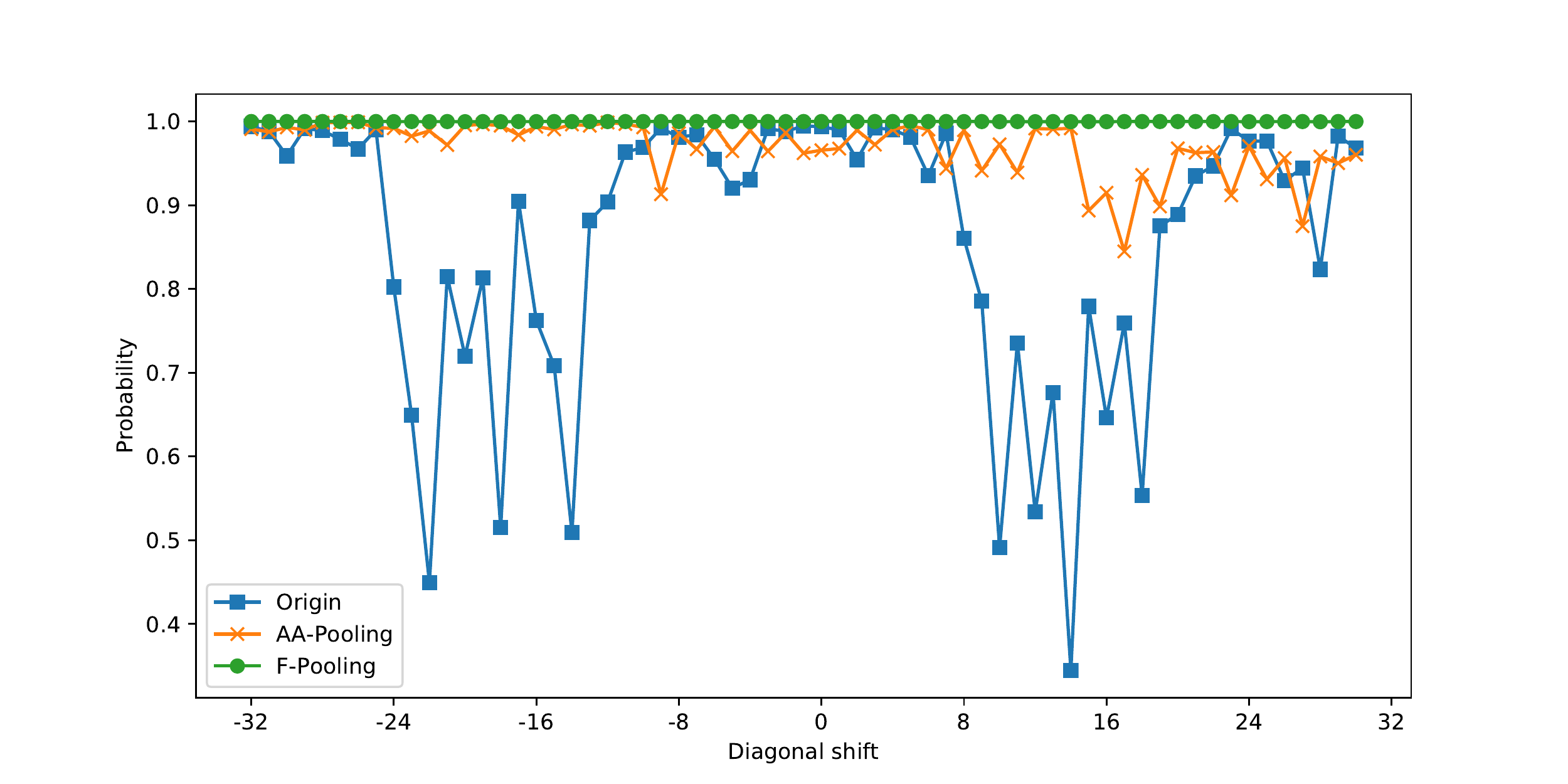}}
	\end{subfigure}
	
	\begin{subfigure}
		\centering
		{\includegraphics[width=0.25\linewidth]{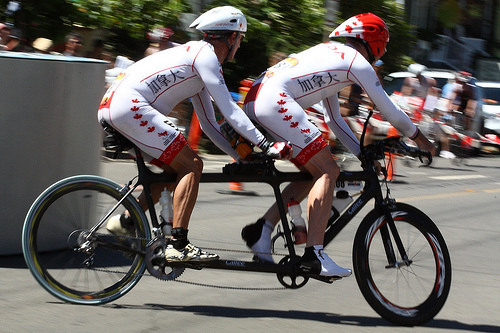}}
		{\includegraphics[width=0.65\linewidth]{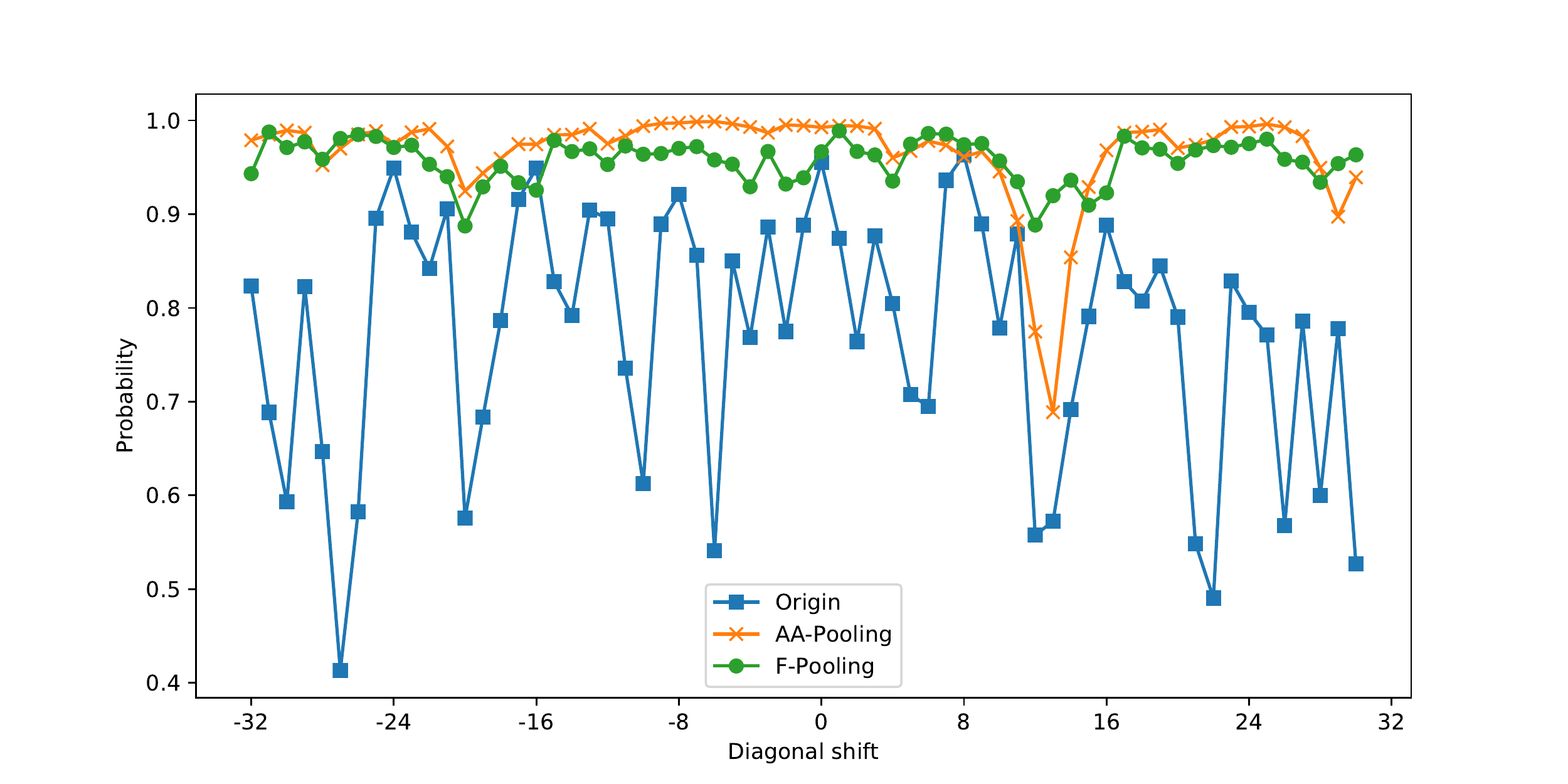}}
	\end{subfigure}
	
	\caption{Classification probability according to the diagonal shift on the images. We use ResNet-50 as the backbone model and integrate different poolings into it. The input images in the left column are selected from the validation set of ImageNet, and the output of CNN varies when we shift the input images. For the right column, the x-coordinate means the pixels of diagonal shift and the y-coordinate means the probabilities of the correct class. F-pooling makes the output more stable than the others with respect to the diagonal shift.}
	
	\label{fig:probabilty}
\end{figure*}

\section{Proposed Method}
In this section, we first define shift-equivalence and anti-aliasing for downsampling. Then, we describe F-pooling in detail and analyze its properties. Finally, we discuss detailed implementation and some practical issues. To simplify the notation, we consider 1D signals when we provide definitions and proofs. It is straightforward to extend them to 2D images.

\subsection{Definition of shift-equivalence}
\label{sec:shift}
Denote $\RR{x}{n}$ and $\RR{y}{m}$ as two 1D signals. We suppose $n > m$. We define
\begin{align}
\mathcal{D}: \mathbf{x} \rightarrow \mathbf{y} \quad & \text{downsampling} \nonumber \\
\mathcal{U}: \mathbf{y} \rightarrow \mathbf{x} \quad &\text{upsampling coupled with } \mathcal{D} \nonumber \\
\mathcal{S}_{\triangle t}: \quad &\text{shift by } \triangle t \text{ pixels} \nonumber 
\end{align}

Then, $\mathcal{D}$ is shift-equivalent if

\begin{equation}
\label{eq_shift}
\boxed
{
	\mathcal{S}_{\triangle t} (\mathcal{UD}\mathbf{x}) \equiv \mathcal{UD} \mathcal{S}_{\triangle t} (\mathbf{x}), \quad \exists \mathcal{U}
}
\end{equation}

If there exists a suitable upsampling $\mathcal{U}$ which makes $\mathcal{U \cdot D}$ and $\mathcal{S}_{\triangle t}$ commutable, then $\mathcal{D}$ is shift-equivalent. That is, applying shift, downsampling and upsampling in order gives the same result as applying upsampling, downsampling and shift in order. Note that the shift operation is required to be circular. When a shifted element hits the edge, it rolls to another side. Unlike the commonly used definition of shift-equivalent such as in \cite{zhang2019making}, an upsampling $\mathcal{U}$ is involved in our definition. We argue that $\mathcal{U}$ is indispensable to make the definition shift-equivalent self-consistent. To see the reason, consider downsampling a signal by factor $s$. Without introducing $\mathcal{U}$, one may define shift-equivalence as:
\begin{eqnarray}
\label{eq_shift_o}
\mathcal{S}_{\triangle t/s}(\mathcal{D}\mathbf{x}) \equiv 
\mathcal{D} \mathcal{S}_{\triangle t} (\mathbf{x})
\end{eqnarray}
However, $\mathcal{S}_{\triangle t/s}$ is not operable for discrete signals when $\triangle t \% s \neq 0$. For example, when $t=1$ and $s=2$, we can not shift the downsampled signal by $0.5$ pixel. To make $\mathcal{S}_{\triangle t/s}$ operable, one should interpolate or upsample the downsampled signals before shifting them, i.e. adding $\mathcal{U}$ at both sides of Eq. \eqref{eq_shift_o}. This leads to the definition of Eq. \eqref{eq_shift}. To our knowledge, it is the first formal definition of shift-equivalence for downsamplings.

As in Eq. \eqref{eq_shift}, to make $\mathcal{D}$ shift-equivalent, one should find the corresponding upsampling $\mathcal{U}$. If such an $\mathcal{U}$ exists, then $\mathcal{U}$ and $\mathcal{D}$ are coupled. The coupled $\mathcal{U}$ plays an important role for proving the shift-equivalence of F-pooling (See Section \ref{sec:F-pooling}).

\subsection{Definition of (anti) aliasing}
Based on classical Nyquist sampling theorem \cite{nyquist1928certain}, the sampling rate must be at least twice as high as the highest frequency of a signal. Otherwise, frequency aliasing appear, i.e. high-frequency components of the signal alias into low-frequency components. This leads to sub-optimal reconstruction results and misleads the following processing because orthogonal components are mixed again. Since CNNs deal with discrete signals, we discuss anti-aliasing for the discrete case. Denote $\mathbf{\hat{x}} \in \mathbb{C}^n$ as the frequency component of $\mathbf{x}$.
\begin{equation}
\mathbf{\hat{x}} = \mathbf{F}_n\mathbf{x}
\end{equation}
where $\mathbf{F}_n \in \mathbb{C}^{n\times n}$ is the so-called DFT-matrix as follows:
\begin{equation}
\mathbf{F}_n ={\begin{bmatrix}\omega _{n}^{0\cdot 0}& \omega _{n}^{0\cdot 1}& \ldots &\omega _{n}^{0\cdot(n-1)}\\
	\omega _{n}^{1\cdot 0}&\omega _{n}^{1\cdot 1}&\ldots &\omega _{n}^{1\cdot (n-1)}\\
	\vdots &\vdots &\ddots &\vdots \\
	\omega _{n}^{(n-1)\cdot 0}&\omega _{n}^{(n-1)\cdot 1}&\ldots &\omega _{n}^{(n-1)\cdot (n-1)}\\
	\end{bmatrix}}
\end{equation}
and $\omega _{n}=e^{-2\pi i/n}$. Based on the property of DFT, the highest frequency of $\mathbf{x}$ is $\lfloor \frac{n}{2} \rfloor$. The sampling rate of a discrete signal is its number of elements.

Now, we define aliasing of downsamplings. Recall $\mathcal{D}: \mathbb{R}^n \rightarrow \mathbb{R}^m$ downsamples a signal into $m$ elements.
\begin{center}
	\noindent\fbox
	{
		\parbox{.9\linewidth}{$\mathcal{D}$ is anti-aliasing if the highest frequency of $\mathcal{D}\mathbf{x}$ is no more than $\lfloor \frac{m}{2} \rfloor$. Otherwise, $\mathcal{D}$ is aliasing.}
	}
\end{center}

\begin{figure*}[t]
	\centering
	\includegraphics[width=0.8\textwidth]{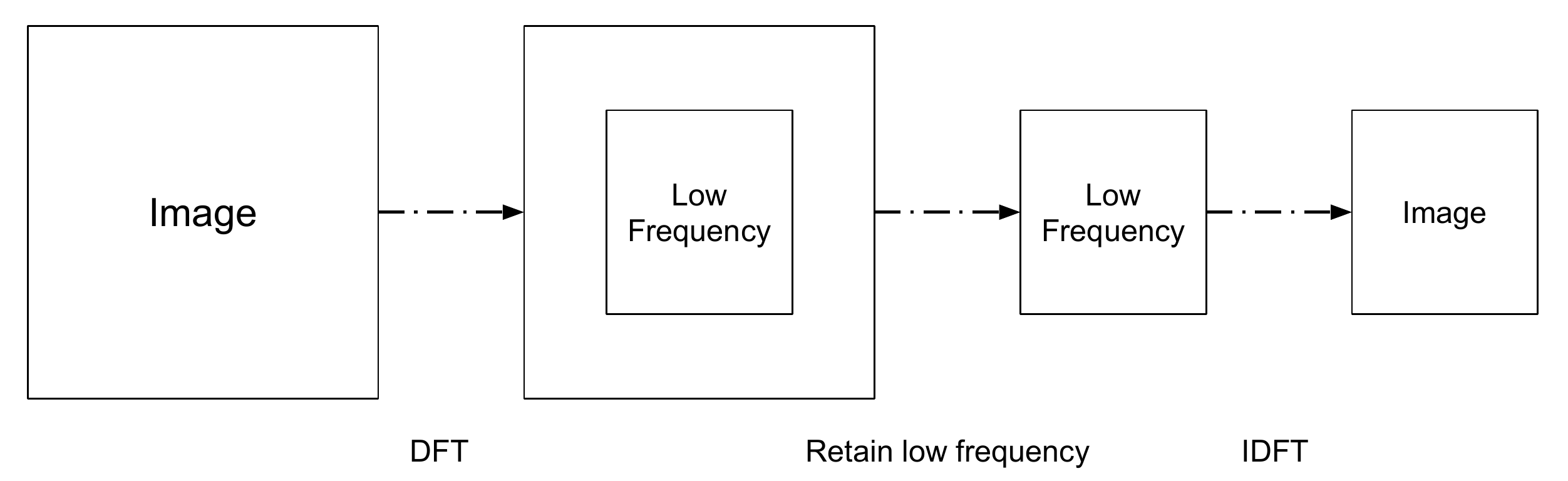}
	\caption{Illustration of the forward process of F-pooling. We assume that the lowest frequency components are located at the center.}
	\label{fig:F-pooling}
\end{figure*}

\subsection{F-pooling}
\label{sec:F-pooling}
The basic idea of F-pooling is to remove the high frequency components of signals and reconstruct the signals only using the low frequency components. In this work, \emph{high frequency components mean the frequencies that are beyond Nyquist frequency}, i.e. half of signal resolution. Moreover, \emph{low frequency components mean the frequencies that are not higher than Nyquist frequency}. To remove high frequency components, we first transform signals into the frequency domain via DFT to retain only the low frequency components. Then, we transform the low frequency components back to the time domain via inverse DFT (IDFT). Fig. \ref{fig:F-pooling} illustrates F-pooling. F-pooling downsamples $n$ elements into $m$ elements. Without loss of generality, we assume $m$ is even. We denote $\mathbf{D}_{\mu}$ as an operation that selects the first $\mu$ rows and the last $\mu$ rows of a matrix.
\begin{equation}
\mathbf{D}_\mu (\mathbf{x})= [\mathbf{x_{1:\mu}}; \mathbf{x}_{N-\mu+1:N}]
\end{equation}
When applying $\mathbf{D}_\mu$ to a signal in the frequency domain, we obtain its lowest $\mu$ frequency basis. Recall that $\mathbf{F}$ is the DFT-matrix and $\mathbf{F}^*$ is the inverse DFT-matrix. Then, the function of F-pooling is represented as:

\begin{equation}
\label{eq:f_pooling}
\boxed
{
	\frac{1}{n} \mathbf{F}_m^*\mathbf{D}_{\frac{m}{2}} \mathbf{F}_n\mathbf{x} \overset{def}{=} \mathbf{Px}
}
\end{equation}
$\mathbf{P} \in \mathbb{C}^{M\times N}$ is the transform matrix of F-pooling.

As mentioned in Section \ref{sec:shift}, the choice of upsampling $\mathcal{U}$ is important. In this work, $\mathcal{U}$ is set to the inverse F-pooling. Specifically, we transform a signal into the frequency domain. Then, we zero-pad the transformed signal to match the resolution of the output. Finally, we transform it back to the time domain. The inverse F-pooling is also represented by matrix multiplications as follows:
\begin{equation}
\label{eq:up}
\boxed
{
	\frac{1}{m} \mathbf{F}_n^* \mathbf{U}_{\frac{m}{2}} \mathbf{F}_m\mathbf{y} \overset{def}{=} \mathbf{\overline{P}y}
}
\end{equation}
where $\mathbf{U}_\mu$ is the zero-padding operation.
\begin{equation}
\mathbf{U}_\mu (\mathbf{y}) = [\mathbf{y_{1:\mu}}; \mathbf{0}; \mathbf{y}_{n-\mu+1:n}]
\end{equation}
and $\overline{\mathbf{P}}$ is the inverse transform matrix of F-pooling.

\begin{figure*}[]
	\centering
	\includegraphics[width=0.9\textwidth]{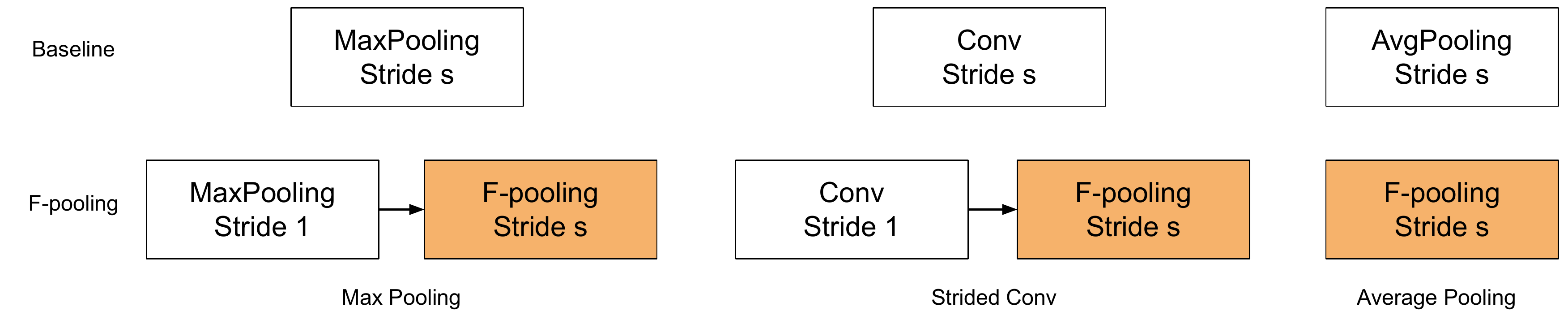}
	\caption{Illustration of replacing max-pooling, average pooling, and stride convolution with F-pooling. We follow the settings in \cite{zhang2019making}.}
	\label{fig:pooling}
\end{figure*}

\subsection{Properties of F-pooling}
\label{sec:prof_1}
In this section, we discuss and prove the properties of F-pooling on the premise that $\mathcal{U}$ is set to $\mathbf{\overline{P}}$.

\subsubsection{Anti-aliasing}
When F-pooling downsamples a signal into $m$ elements, it first remove the frequency components of the signal that are higher than $\lfloor \frac{m}{2} \rfloor$. Thus, F-pooling is anti-aliasing by definition. Note that the previous anti-aliasing pooling \cite{zhang2019making} applies low-pass filtering before downsampling. Although its idea is similar to F-pooling, it is not anti-aliasing in theory since the low-pass filter used in \cite{zhang2019making} cannot remove high frequency components completely.

\subsubsection{Optimal reconstruction}
Prior to discussing the optimal reconstruction, we first prove that applying F-pooling and its inverse in order keeps the low frequency components. The low frequency component $\mathbf{x}_l$ and high frequency component $\mathbf{x}_h$ are defined as:
\begin{equation}
\boxed
{
	\mathbf{x}_l \overset{def}{=} \frac{1}{n}\mathbf{F}_n^*\mathbf{L}_{\frac{m}{2}}\mathbf{F}_n \mathbf{x}, 
	\quad \mathbf{x}_h \overset{def}{=} \mathbf{x} - \mathbf{x}_l
}
\end{equation}
where $\mathbf{L}_\mu$ keeps the first and last $\mu$ rows of a matrix by setting other rows to zero. Thus $\mathbf{L}_\mu$ is a diagonal matrix whose elements are $0$ or $1$.

\begin{lemma}
	\label{the_1}
	$\mathbf{\overline{P}Px} = \mathbf{x}_l$.
\end{lemma}

\begin{proof}
	\begin{align}
	\label{eq:prof_3}
	\mathbf{\overline{P}Px} & = \frac{1}{mn} \mathbf{F}_n^* \mathbf{U}_{\frac{m}{2}} \mathbf{F}_m
	\mathbf{F}_m^*\mathbf{D}_{\frac{m}{2}} \mathbf{F}_n  \mathbf{x} \\ \nonumber
	& = \frac{1}{n} \mathbf{F}_n^* \left( \mathbf{U}_{\frac{m}{2}} \mathbf{D}_{\frac{m}{2}}  \right) \mathbf{F}_n  \mathbf{x} \\ \nonumber
	& = \frac{1}{n} \mathbf{F}_n^* \mathbf{L}_{\frac{m}{2}} \mathbf{F}_n  \mathbf{x} = \mathbf{x}_l
	\end{align}
	The last line of \eqref{eq:prof_3} holds by the definitions of $\mathbf{U}$, $\mathbf{D}$, and $\mathbf{L}$. Specifically, $\mathbf{D}$ can be represented as an $m\times n$ matrix and $\mathbf{U} = \mathbf{D}^T$.
	\begin{equation}
	\mathbf{D}_{\frac{m}{2}} ={\begin{bmatrix}
		\mathbf{I}_{\frac{m}{2}} & \cdots & \mathbf{0} & \cdots &\mathbf{0}_{\frac{m}{2}} \\
		\mathbf{0}_{\frac{m}{2}} & \cdots & \mathbf{0} &\cdots &\mathbf{I}_{\frac{m}{2}} \\
		\end{bmatrix}}
	\end{equation}
	where $\mathbf{I}_{m}$ is an $m\times m$ identity matrix and $\mathbf{0}_{m}$ is an $m\times m$ zero matrix. Then, we have
	\begin{equation}
	\mathbf{U}_{\frac{m}{2}} \mathbf{D}_{\frac{m}{2}} ={\begin{bmatrix}
		\mathbf{I}_{\frac{m}{2}} & \mathbf{0} & \mathbf{0} \\
		\mathbf{0} & \mathbf{0}_{n-m}& \mathbf{0} \\
		\mathbf{0} & \mathbf{0} & \mathbf{I}_{\frac{m}{2}} \\
		\end{bmatrix}}
	= \mathbf{L}_{\frac{m}{2}}
	\end{equation}
	
	$\hfill\blacksquare$ 
\end{proof}

Now, we prove that F-pooling is the optimal anti-aliasing downsampling from the perspective of reconstruction. 

\begin{theorem}
	F-pooling $\mathbf{P}$ is the optimal anti-aliasing downsampling that minimizes the following objective:
	\begin{equation}
	\label{eq:obj_1}
	\min_\mathcal{D} || \mathbf{\overline{P}}\mathcal{D}\mathbf{x} - \mathbf{x} ||_2^2, \quad s.t. \quad \mathcal{D} \text{ is anti-aliasing}
	\end{equation}
\end{theorem}

\begin{proof}
	\begin{align}
	\label{eq:prof_2}
	|| \mathbf{\overline{P}}\mathcal{D}\mathbf{x} - \mathbf{x} ||_2^2 & = || \mathbf{\overline{P}}\mathcal{D}\mathbf{x} - \mathbf{x}_l - \mathbf{x}_h ||_2^2 \\ \nonumber
	\label{eq:prof_1}
	& = ||\mathbf{\overline{P}}\mathcal{D}\mathbf{x} - \mathbf{x}_l ||_2^2 + || \mathbf{x}_h ||_2^2 +\\ \nonumber
	&  \quad \left< \mathbf{\overline{P}}\mathcal{D}\mathbf{x}, \mathbf{x}_h \right> - \left< \mathbf{x}_l, \mathbf{x}_h \right> \\ \nonumber
	& = || \mathbf{\overline{P}}\mathcal{D}\mathbf{x} - \mathbf{x}_l ||_2^2 + || \mathbf{x}_h ||_2^2
	\end{align}
	Due to the orthogonality of frequency spectrum,
	$\left< \mathbf{x}_l, \mathbf{x}_h \right> = 0$. Because $\mathcal{D}$ is anti-aliasing here, $\mathcal{D}\mathbf{x}$ only contains low frequencies of $\mathbf{x}$. Also because $\mathbf{\overline{P}}$ introduces no extra frequencies, $\mathbf{\overline{P}}\mathcal{D}\mathbf{x}$ only contains low frequencies of $\mathbf{x}$. Then, we have $\left< \mathbf{\overline{P}}\mathcal{D}\mathbf{x}, \mathbf{x}_h \right> = 0$. This reveals that Eq. \eqref{eq:prof_2} holds.
\end{proof}

Now, minimizing the reconstruction error is equivalent of minimizing:
$|| \mathbf{\overline{P}}\mathcal{D}\mathbf{x} - \mathbf{x}_l ||_2^2$. Based on lemma \ref{the_1}, $|| \mathbf{\overline{P}P}\mathbf{x} - \mathbf{x}_l ||_2^2 = 0$. Thus  the reconstruction error is minimized when $\mathcal{D} = \mathbf{P}$. We have proved that F-pooling is the optimal anti-aliasing downsampling from the perspective of reconstruction. $\hfill\blacksquare$ 

Although it is difficult to define the optimality for an intermediate layer, we can assume that the reconstruction optimality is related to the final performance, e.g. accuracy for image classification. Prior works have shown that self reconstruction of intermediate features is helpful for image classification \cite{rasmus2015semi-supervised}. Thus, it is expected to maintain intermediate features for poolings as much as possible. This means the optimal reconstruction property of F-pooling is helpful for image classification.

\subsubsection{Shift-equivalence}
\begin{theorem}
	F-pooling $\mathbf{P}$ is shift-equivalent.
\end{theorem}
\begin{proof}
	\label{sec:prof_2}
	Based on the shift theorem of Fourier transform, we have
	\begin{align}
	\mathbf{F}_n \mathcal{S}_{\triangle t} (\mathbf{x}) = & \mathbf{F}_n \mathbf{x} \odot \mathbf{S}_{\triangle t} \\
	\mathcal{S}_{\triangle t} (\mathbf{F}_n^* \mathbf{x}) = & \mathbf{F}_n^* \mathbf{x} \odot \mathbf{S}_{\triangle t}
	\end{align}
	where $\mathbf{S}_{\triangle t} \in \mathbb{C}^n$ whose $k$th element is $e^{-\frac{2\pi k}{n}\triangle t}$ and $\odot$ is element-wise multiplication. Combining with \eqref{eq:prof_3}, we have
	\begin{align}
	\label{eq:prof_4} 
	\mathbf{\overline{P}P}\mathcal{S}_{\triangle t}(\mathbf{x}) & = \frac{1}{n} \mathbf{F}_n^* \mathbf{I}_{\frac{m}{2}} \mathbf{F}_n \mathcal{S}_{\triangle t}(\mathbf{x}) \\ \nonumber
	& = \frac{1}{n} \mathbf{F}_n^* \left( \mathbf{I}_{\frac{m}{2}} \mathbf{F}_n \mathbf{x} \right) \odot \mathbf{S}_{\triangle t} \\ \nonumber
	& = \mathcal{S}_{\triangle t} \left( \frac{1}{n}\mathbf{F}_n^* \mathbf{I}_{\frac{m}{2}} \mathbf{F}_n \mathbf{x} \right) \\ \nonumber
	& = \mathcal{S}_{\triangle t} (\mathbf{\overline{P}P}\mathbf{x})
	\end{align}
	Based on the definition in Section \ref{sec:shift}, we prove that F-pooling is shift-equivalent. $\hfill\blacksquare$
\end{proof}

\begin{figure}[]
	\centering
	\includegraphics[width=0.45\textwidth]{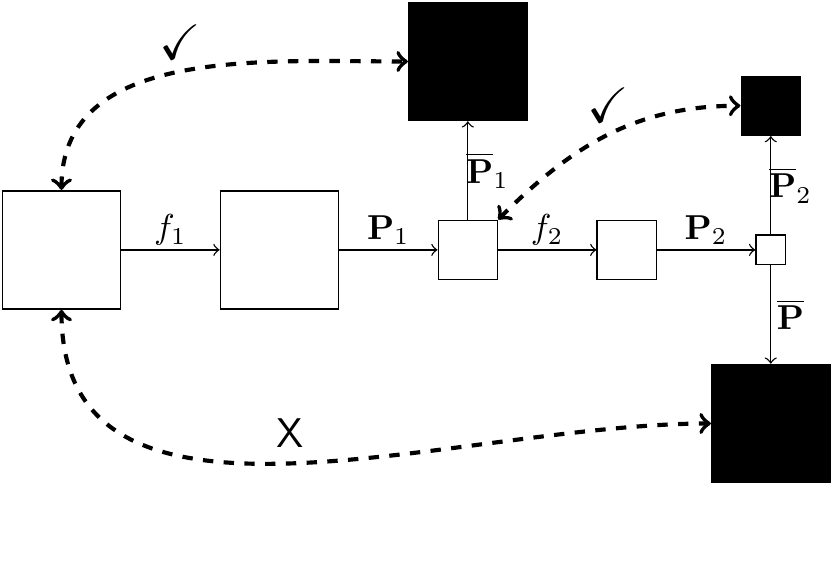}
	\caption{Transitivity of shift-equivalence in two stacked F-poolings. $f_i$ is a shift-equivalent function which keeps the resolution of features. The hollowed rectangles indicate the features of a CNN while the solid rectangles indicate the features which are upsampled by inverse F-pooling. $\checkmark$ means two features are shift-equivalent, while X means they are not shift-equivalent.} 
	\label{fig:shift_t}
\end{figure}

\begin{figure}[]
	\centering
	\includegraphics[width=0.45\textwidth]{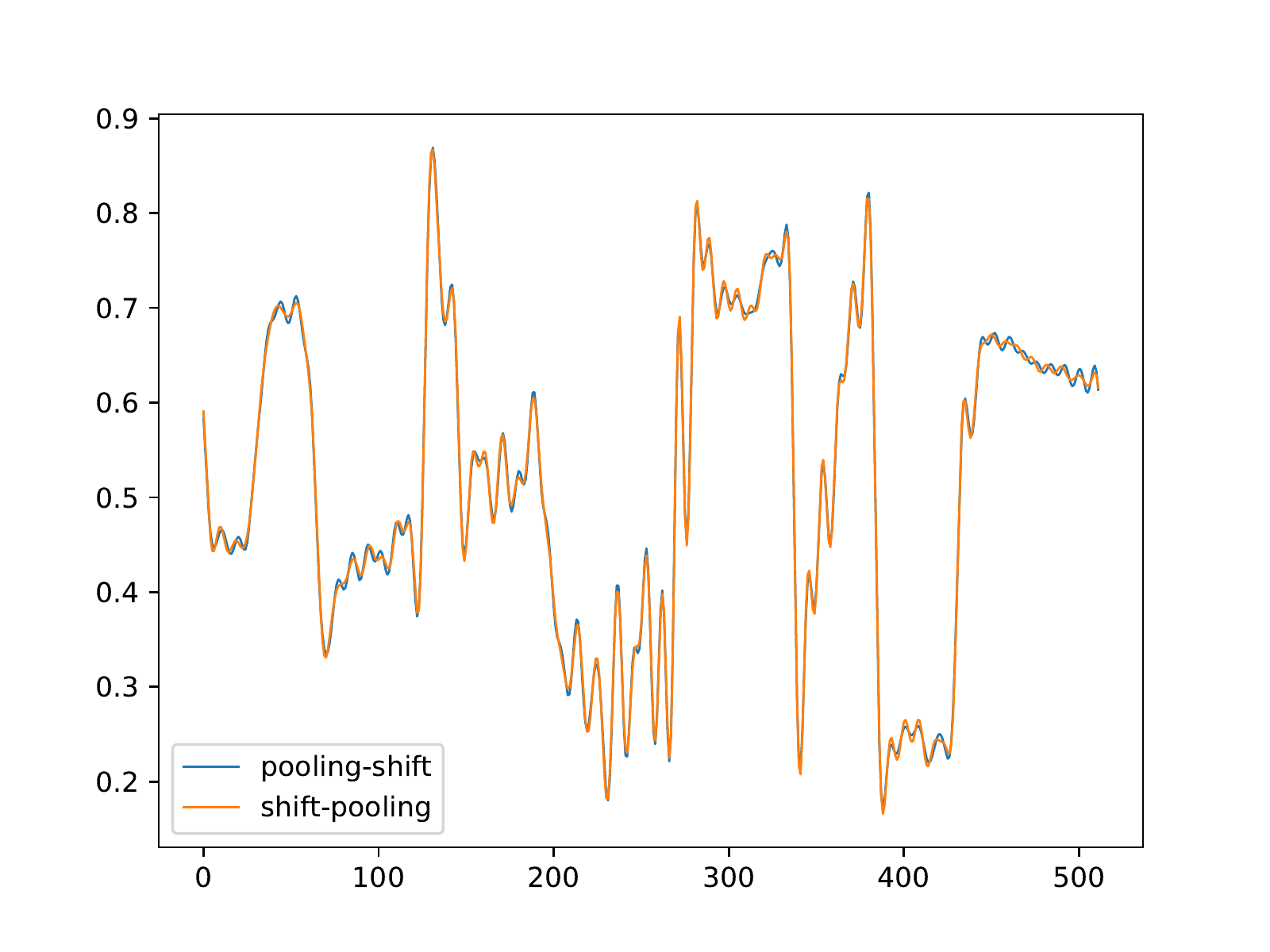}
	\caption{Without odd padding, F-pooling is not strict shift-equivalent. However, the error is acceptable. Best viewed on screen.}
	\label{fig:odd}
\end{figure}

\begin{table*}[t]
	\caption{Accuracy and consistency on CIFAR-100}
	\centering
	\begin{tabular}{ccccc}
		\hline
		Accuracy/Consistency & Shift argument & DenseNet-41 & ResNet-18 & ResNet-34\\
		\hline
		\hline
		Origin             & with & 74.90/71.55 &75.52/70.21& 76.56/72.21\\
		AA-pooling     & with & \textbf{75.55}/71.71 &\textbf{77.43}/73.08 & 76.95/73.38\\
		F-pooling        & with & 75.45/\textbf{71.91} &77.36/\textbf{73.43} & \textbf{77.68}/\textbf{73.54}\\
		\hline
		Origin             & without & 71.81/57.27 & 67.60/45.00 & 68.11/46.56\\
		AA-pooling     & without & \textbf{73.81}/60.29 & \textbf{74.49}/\textbf{58.24} & \textbf{74.00}/\textbf{57.72}\\
		F-pooling        & without & 73.19/\textbf{60.51} & 73.93/57.48 & 73.51/56.54\\
		\hline
	\end{tabular}	
	\label{tab:cifar}
\end{table*}

\subsubsection{Transitivity of shift-equivalent} We study the transitivity of shift-equivalent for F-pooling. Denote $f$ as a shift-equivalent function which keeps the resolution of features, i.e. no downsampling and upsampling are involved. We consider whether the composition of F-pooling and $f$ is shift-equivalent.

\begin{theorem}
	The function $\mathbf{P}f$ is shift-equivalent.
\end{theorem}
\begin{proof}
	\label{sec:prof_3}
	\begin{align}
	\mathcal{S}_{\triangle t}\overline{\mathbf{P}} \overbrace{\mathbf{P}(f}\mathbf{x}) = &
	\overline{\mathbf{P}}\mathbf{P}\mathcal{S}_{\triangle t}(f\mathbf{x}) \\ \nonumber
	= & \overline{\mathbf{P}} \overbrace{\mathbf{P}f} \mathcal{S}_{\triangle t}\mathbf{x}
	\end{align}
	
	The first line holds because F-pooling is shift-equivalent and the second line holds because $f$ is shift-equivalent.  $\hfill\blacksquare$
\end{proof}

\begin{table*}[t]
	\centering
	\caption{Accuracy and consistency on sub-ImageNet}
	\begin{tabular}{cccc}
		\hline
		Accuracy/Consistency & DenseNet-121 & ResNet-50 & Mobilenet-v2\\
		\hline
		\hline
		Origin             & 77.47/82.74 &74.44/80.17 & 73.01/78.08\\
		AA-pooling     & 77.14/83.09 &\textbf{76.12}/82.26 & 74.03/79.12\\
		F-pooling        & \textbf{77.56}/\textbf{84.10} &76.05/\textbf{82.63} & \textbf{74.72}/\textbf{80.34}\\
		\hline
	\end{tabular}	
	\label{tab:imagenet}
\end{table*}

That is, shift-equivalence of F-pooling is transitive for $f$. Note that $f$ is allowed to be non-linear. For CNNs, $f$ could be arbitrary combinations of convolutional layers, normalization layers and activation layers. For example, a function of $conv \rightarrow relu \rightarrow \mathbf{P}$ is still shift-equivalent. Not only the final results but also the intermediate results of $f$ should maintain the resolution of inputs.

Is a function that stacks multiple F-poolings still shift-equivalent? Unfortunately, the answer is no. When we define shift-equivalence of a downsampling, we should introduce its coupled upsampling. F-pooling and its inverse are coupled in the definition. That is, $\overline{\mathbf{P}}$ should immediately appear behind $\mathbf{P}$. For a function with multiple F-poolings, one should find a suitable upsampling at its final results. If we choose inverse F-pooling as the upsampling, shift-equivalence does not hold because it is not coupled with any intermediate F-poolings of the composite downsampling. We provide this phenomenon in Fig. \ref{fig:shift_t}. 

In summary, a CNN is locally shift-equivalent when we consider two features where there is zero or one F-pooling between them. When there are two or more F-poolings, a CNN as a whole is no longer shift-equivalent. Although shift-equivalence is not transitive to multiple F-poolings, we find that CNNs with F-poolings are more robust to the image shifts than the ones with commonly used poolings. 

\subsection{Practical issues}
\label{sec:practical}

\subsubsection{Imaginary part} In general, the output of F-pooling $\mathbf{P}\mathbf{x}$ contains both real part and imaginary part. However, for commonly used CNNs, the feature maps should be real. Thus, one need ignore the imaginary part of F-pooling. On the other hand, $\mathbf{P}\mathbf{x}$ is treated as complex when we prove the properties of F-pooling. If we ignore the imaginary part, those properties may not hold, i.e. F-pooling may be no longer shift-equivalent.

Let's analyze the imaginary part in detail. Recall that $m$ is the resolution of downsampled signal. If $m$ is odd, i.e. $m=2\mu - 1$, the intermediate result of F-pooling before back to the time domain is:
\begin{equation}
\label{eq:odd}
\mathbf{D}_{\frac{m}{2}}\mathbf{F}_n \mathbf{x} = 
[\hat{\mathbf{x}}_0, \dots, \hat{\mathbf{x}}_{\mu-1}, \hat{\mathbf{x}}_{-\mu+1}, \dots, \hat{\mathbf{x}}_{-1}]
\end{equation}
Due to the symmetry of DFT, the imaginary part is zero when transforming it back to the time domain. Thus, the imaginary part of F-pooling can be safely ignored when $m$ is odd. If $m$ is even, i.e. $m=2\mu$, the intermediate result is
\begin{equation}
\label{eq:even}
[\hat{\mathbf{x}}_0, \dots, \hat{\mathbf{x}}_{\mu-1}, \boxed{\hat{\mathbf{x}}_{\mu}}, \hat{\mathbf{x}}_{-\mu+1}, \dots, \hat{\mathbf{x}}_{-1}]
\end{equation}
In this case, it contains imaginary part. Compared with \eqref{eq:odd}, the imaginary part is introduced by the non-symmetric term $\mathbf{\hat{x}}_{\mu}$. We can recover the symmetry and eliminate the imaginary part by setting $\hat{\mathbf{x}}_{\mu}$ to zero. We call this trick \emph{odd padding}. With odd padding, F-pooling becomes shift-equivalent again. However, more information is lost during downsampling which may reduce the accuracy of CNNs, especially when the resolution is small. We provide the effects of odd padding in Fig. \ref{fig:odd}. Although F-pooling is not shift-equivalent without odd padding, the error is acceptable in practice. For a better trade-off, we \emph{do not} use odd padding in this work.

\subsubsection{Padding matters} F-pooling is designed to be circular shift-equivalence. When combining with convolutions, their padding mode should be circular padding. However, zero padding is commonly used in convolution layers which destroys their circular shift-equivalence. In this work, we evaluate the performance of F-pooling with both zero-padding and circular padding. As expected, we find that circular padding is better than zero-padding from the perspective of shift-equivalence. However, circular padding is slower than zero padding on open-source deep learning libraries.

\begin{table}[]
	\centering
	\caption{Number of poolings on sub-ImageNet.}
	
	\begin{tabular}{cccc}
		\hline
		Pooling & Max & Average & Strided-conv \\
		\hline
		\hline
		DenseNet          & 1 & 3 & 1 \\
		ResNet                & 1 & 0 & 4 \\		
		MobileNet-v2  & 1 & 0 & 4 \\
		\hline
	\end{tabular}	
	\label{tab:pooling}
\end{table}

\subsection{Implementation}
As 2D DFT is represented as two 1D DFTs along the vertical and horizontal directions, 2D F-pooling can be represented as two 1D F-poolings. 
\begin{equation}
\mathbf{Y} = \mathbf{PXP}^T
\end{equation}
where $\mathbf{X}$ is a $n \times n$ matrix and $\mathbf{Y}$ an $m \times m$ matrix. Because two 1D F-poolings are applied separately, the proofs and properties in the previous sections still hold for 2D signals. We implement F-pooling in PyTorch \cite{paszke2017automatic}. For computational complexity, it is the best to implement F-pooling via fast Fourier transform (FFT) and its inverse. When we downsample $\mathbf{X}$ to $\mathbf{Y}$, the time complexity of F-pooling is $\mathcal{O} (n^2\log n)$. Unfortunately, we find such an implementation in PyTorch is not as fast as expected.

Instead, we implement F-pooling via matrix multiplications. In this way, the time complexity is $\mathcal{O}(n^2m)$. F-pooling requires more computational costs than average pooling or max-pooling. It is faster than a convolutional layer when the number of channels is large like modern CNNs. Moreover, the number of pooling layers is limited compared with the number of convolutional layers. Thus, F-pooling results in a little computational cost into CNNs. Zhang \cite{zhang2019making} claims that it is important to integrate anti-aliasing pooling into CNNs in a proper way. In this work, we follow his settings for a fair comparison. We decompose a pooling with downsampling factor $s$ into two parts: a pooling with factor $1$ and an F-pooling with factor $s$. As shown in Fig. \ref{fig:pooling}, a max-pooling with stride $s$ is replaced with a max-pooling with stride $1$ and a F-pooling with factor $s$. A convolution with stride $s$ is replaced with a convolution with stride $1$ and an F-pooling with factor $s$. An average pooling with stride $s$ is replaced with an F-pooling with factor $s$.

\begin{table*}[]
	\caption{Standard deviation (std) and consistency on sub-ImageNet with circular padding.}
	\centering
	\begin{tabular}{cccc}
		\hline
		Std/Consistency & DenseNet-121 & ResNet-50 & MobileNet-v2\\
		\hline
		\hline
		Origin             & 0.043/90.44 &0.051/87.97 & 0.051/87.71\\
		AA-pooling     & 0.055/88.21 &0.056/87.77 & 0.059/86.19\\
		
		F-pooling        & \textbf{0.035}/\textbf{91.88} &\textbf{0.037}/\textbf{91.01} & \textbf{0.041}/\textbf{90.32}\\
		\hline
	\end{tabular}	
	\label{tab:imagenet_circular}
\end{table*}

\begin{table*}[t!]
	\caption{Standard deviation (std) and consistency on CIFAR-100 with circular padding.}
	\centering
	\begin{tabular}{cccc}
		\hline
		Std/Consistency & DenseNet-41 & ResNet-18 & ResNet-34\\
		\hline
		\hline
		Origin             & 0.111/77.62 & 0.176/59.93 & 0.183/58.94\\
		AA-pooling     & 0.166/66.97 & 0.164/65.67 & 0.180/60.53\\
		F-pooling        & \textbf{0.088}/\textbf{82.99} & \textbf{0.073}/\textbf{84.34} & \textbf{0.088}/\textbf{80.90}\\
		\hline
	\end{tabular}
	\label{tab:cifar_circular}
\end{table*}

\section{Experimental Results}

\subsection{1D signals}
We test the shift-equivalence of F-pooling on 1D signals. Given a signal, we first apply pooling, upsampling, and shift sequentially to them. Then, we apply shift, pooling, and upsampling sequentially to it. In this way, we obtain two transformed signals and compare whether they are exactly the same. In Fig. \ref{fig:test}, the original signal is a randomly selected row of a $512 \times 512$ image. We apply max-pooling, average pooling, and F-pooling with stride 4 to those signals. As shown in Fig. \ref{fig:test}, F-pooling is perfectly shift-equivalent. Average pooling performs better than max-pooling from the perspective of shift-equivalence. We use \emph{odd padding} for the experiments.

\begin{table*}[t]
	\caption{Accuracy and consistency on CIFAR-100 with fewer frequencies.}
	\centering
	\begin{tabular}{lccc}
		\hline
		Accuracy/Consistency & DenseNet-41 & ResNet-18 & ResNet-34\\
		\hline
		\hline
		$25.0\%$                                 & 60.55/54.47 & 62.62/57.44 & 61.16/54.80\\
		$37.5 \%$                                & 71.91/57.00 & 73.45/\textbf{57.61} & 73.30/\textbf{57.57}\\
		$50.0\%$(F-pooling)          & \textbf{73.19}/\textbf{60.51} & \textbf{73.93}/57.48 & \textbf{73.51}/56.54\\
		\hline
	\end{tabular}
	\label{tab:rec}
\end{table*}

\subsection{Image classification}
\label{sec:IC}
We evaluate accuracy and shift robustness of F-pooling on image classification. We measure shift robustness by the consistency of classifications when we shift images, the same as \cite{zhang2019making}. We check how often the model outputs the same class given the same image with two different shifts.
\begin{equation}
\mathbb{E}_{\triangle t_0, \triangle  t_1} \mathbf{1} \left\{ C(\mathcal{S}_{\triangle t_0}) =C(\mathcal{S}_{\triangle t_1}) \right\}
\end{equation}
where $C$ is the predicted category. For convenience, we only evaluate diagonal shifts in this work. The padding mode of convolutions is set to zero padding. We compare F-pooling with AA-pooling. \emph{Origin} means that we keep the pooling methods of CNNs.

\subsubsection{CIFAR-100}
CIFAR-100 \cite{2009Learning} contains 50k images for training and 10k images for test. Each sample is a low-resolution $32 \times 32$ color image that is classified into one of 100 categories. We train this dataset on ResNet \cite{he2016deep} and DenseNet \cite{huang2017densely}. The number of shifted pixels ranges from -7 to 7 when we evaluate the consistency. We show the accuracy and consistency in Table \ref{tab:cifar}. Those results are averaged by 3 runs. F-pooling is better than commonly poolings and it is competitive to AA-pooling in terms of accuracy and consistency. Moreover, F-pooling is complementary to shift data augmentation.

\subsubsection{Sub-ImageNet}
We evaluate the classification performance on high-resolution images. Original ImageNet dataset \cite{ILSVRC15} contains 1.28M training and 50k validation images, classified into one of 1,000 categories. To reduce the computational resources for training, we use a subset of ImageNet, called \emph{sub-ImageNet}. We randomly select 200 categories from 1,000 categories. For each category, we randomly select 500 images as the training set. Thus, we collect 100k images for training. We select corresponding validation images belonging to the selected categories. There are about 10k validation images. All models are trained on a single GPU with batch size 64 and 100 epochs. We decrease the initial learning rate by a factor 10 every 40 epochs. For other hyper-parameters, we follow the official training script of PyTorch \footnote{https://github.com/pytorch/examples/tree/master/imagenet}. We train ResNet, DenseNet, and MobileNet-v2 \cite{sandler2018mobilenetv2}. Those models are widely used as benchmarks. As shown in Table \ref{tab:pooling}, max-pooling, average pooling, and strided-convolution are covered in them. As in \cite{zhang2019making}, we keep the first downsampling of models to reduce computational costs. The number of shifted pixels ranges from -63 to 63 when we evaluate the consistency. We provide the accuracy and consistency in Table \ref{tab:imagenet}. F-pooling performs consistently better than common poolings, which is slightly better than AA-pooling on sub-ImageNet. We also provide the loss curves in Fig. \ref{fig:loss}. As shown in the figure, the training cross-entropy of F-pooling is not lower than the others, whereas the test corss-entropy of F-pooling is slightly lower than the others.

\begin{figure*}[t!]
	\centering
	\subfigure[DenseNet-train]{
		\includegraphics[width=0.4\textwidth]{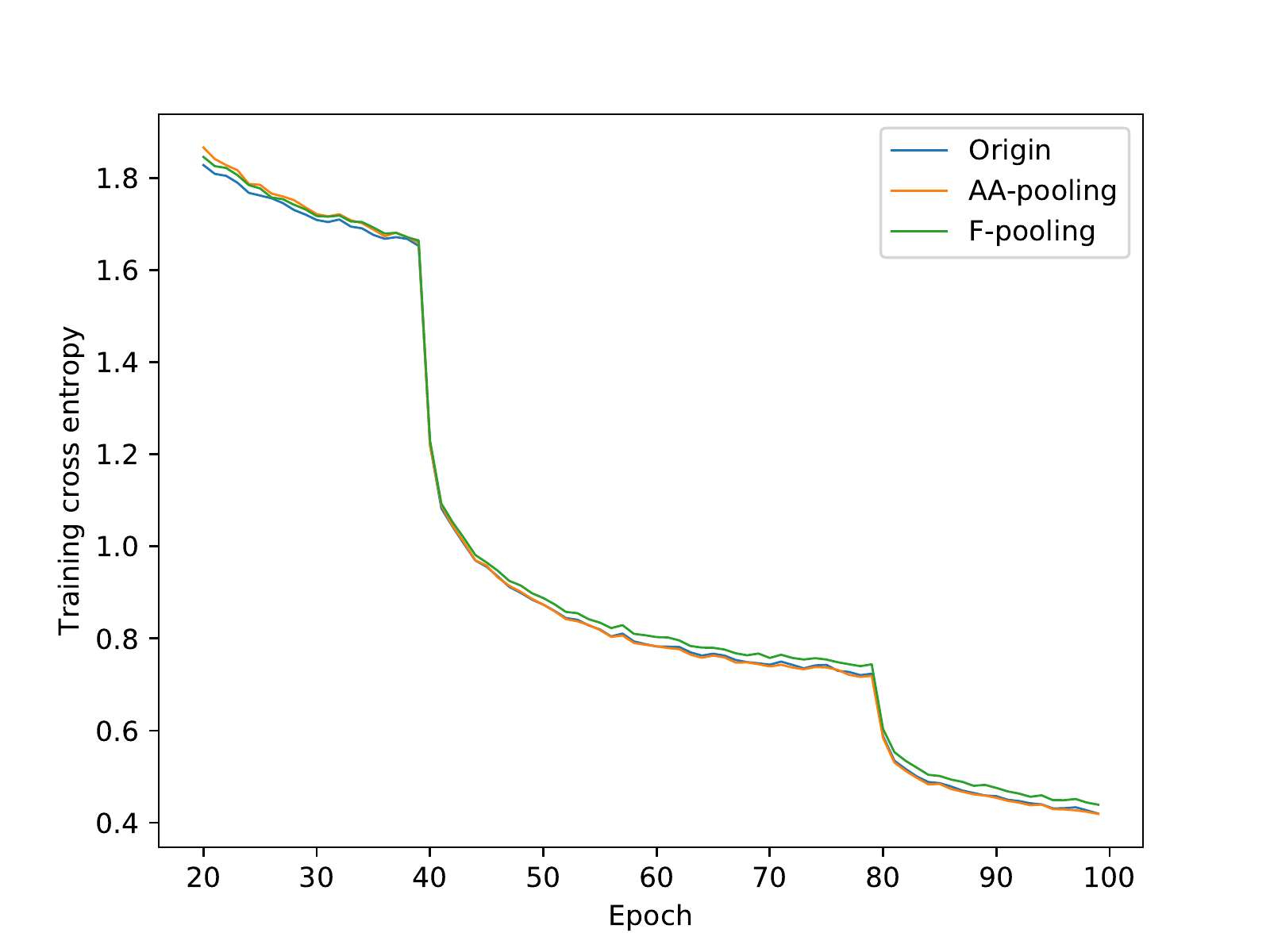}
	}
	\subfigure[DenseNet-test]{
		\includegraphics[width=0.4\textwidth]{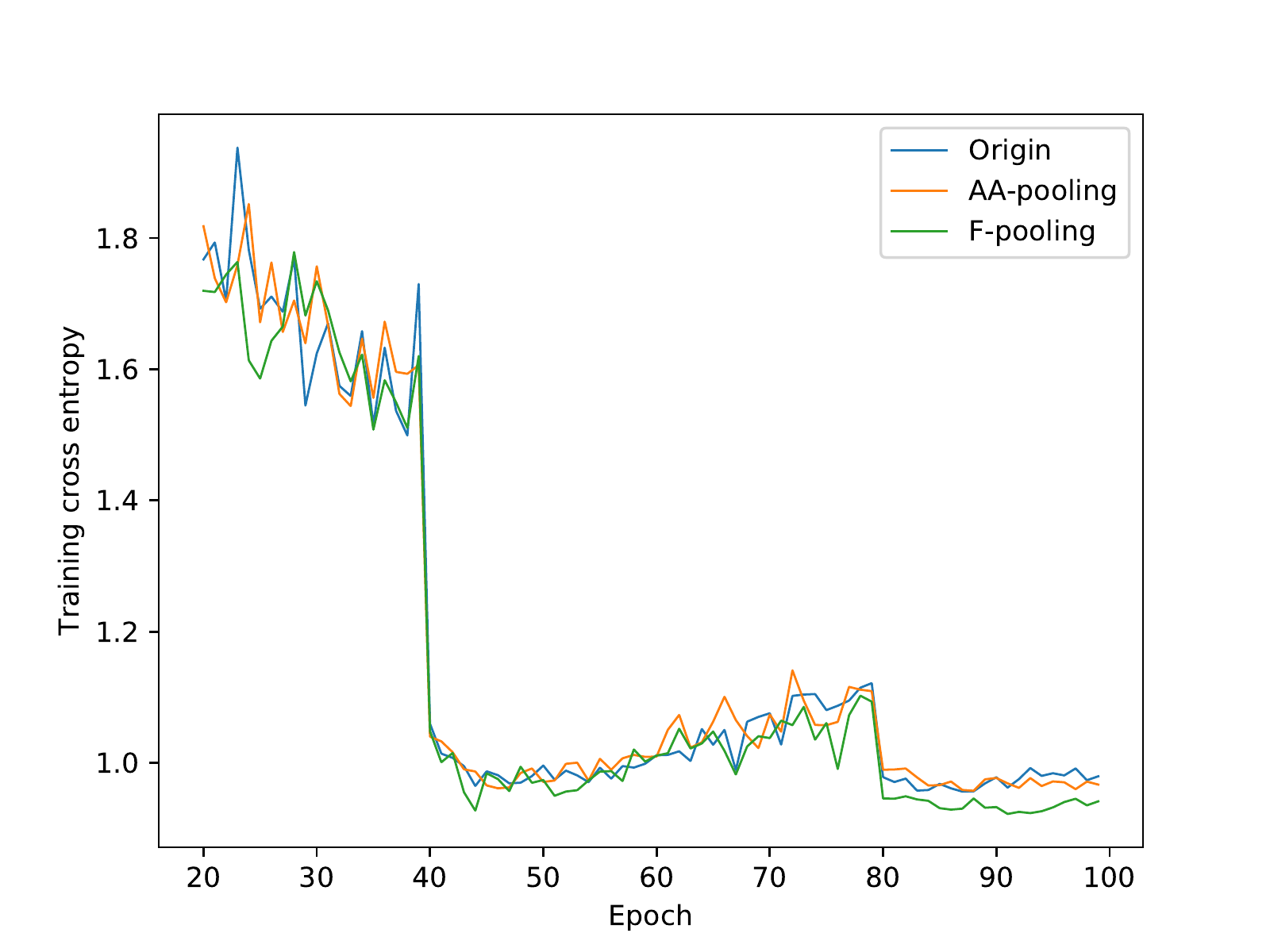}
	}
	\subfigure[ResNet-train]{
		\includegraphics[width=0.4\textwidth]{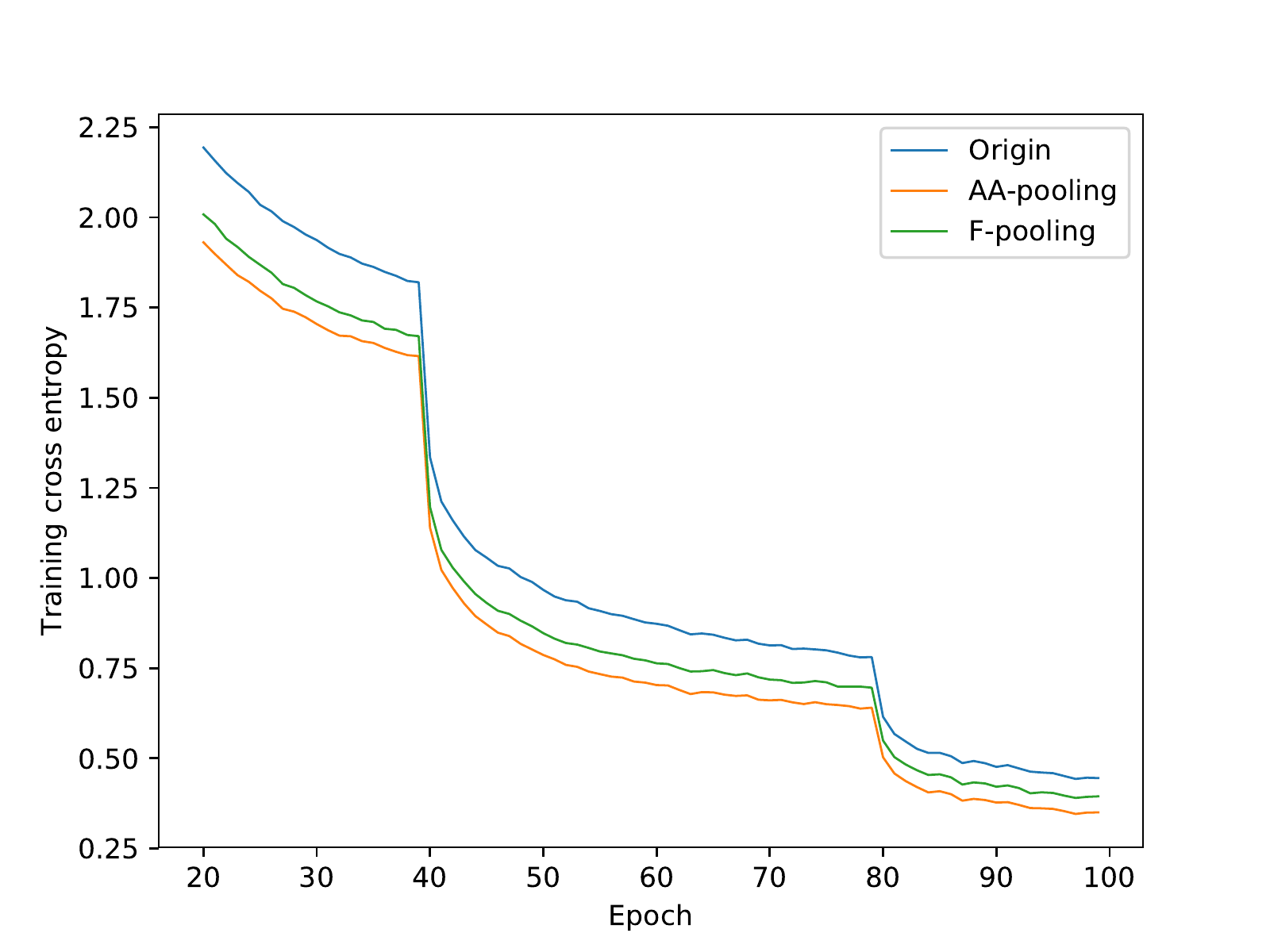}
	}
	\subfigure[ResNet-test]{
		\includegraphics[width=0.4\textwidth]{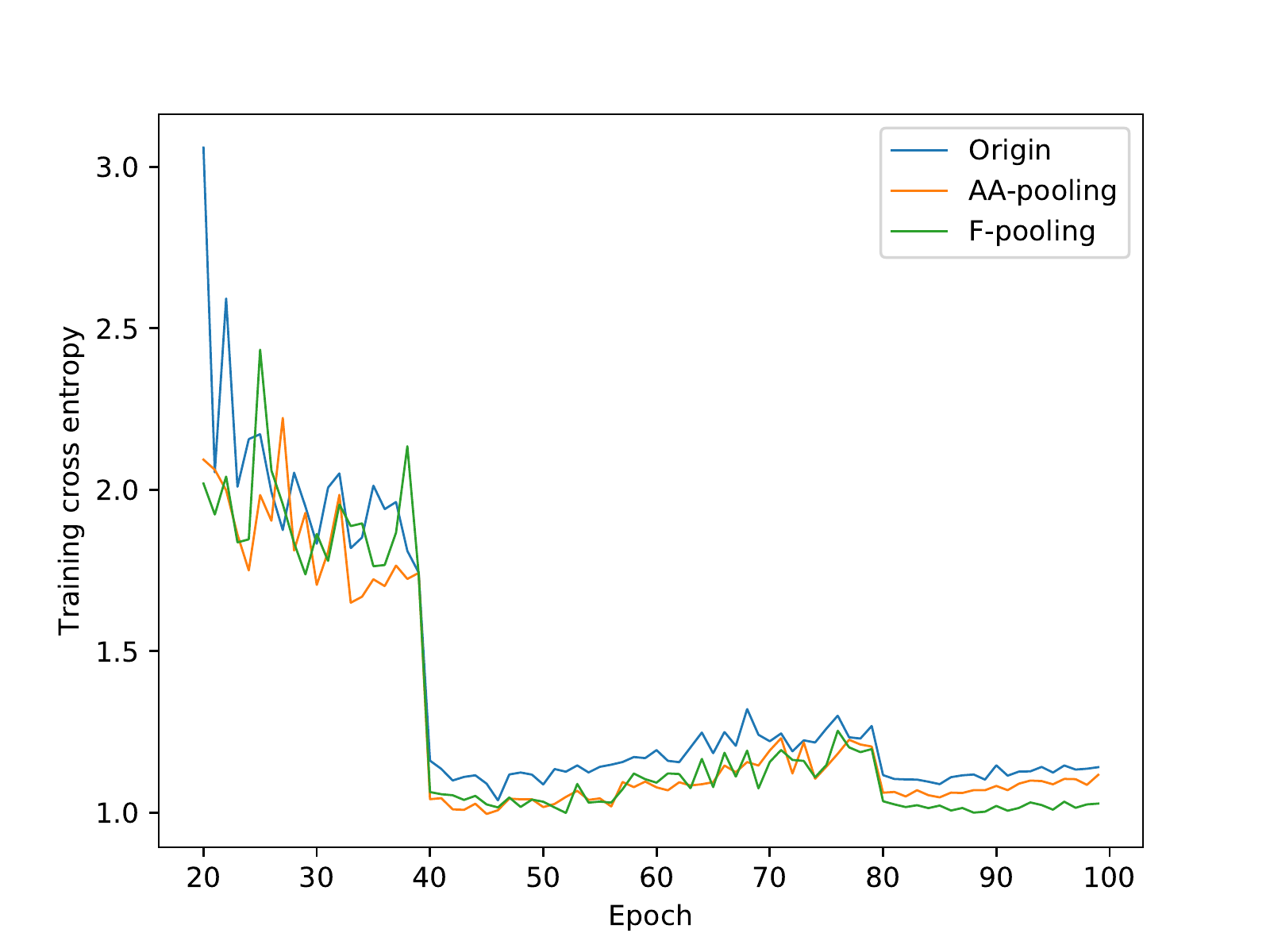}
	}
	\subfigure[MobileNet-train]{
		\includegraphics[width=0.4\textwidth]{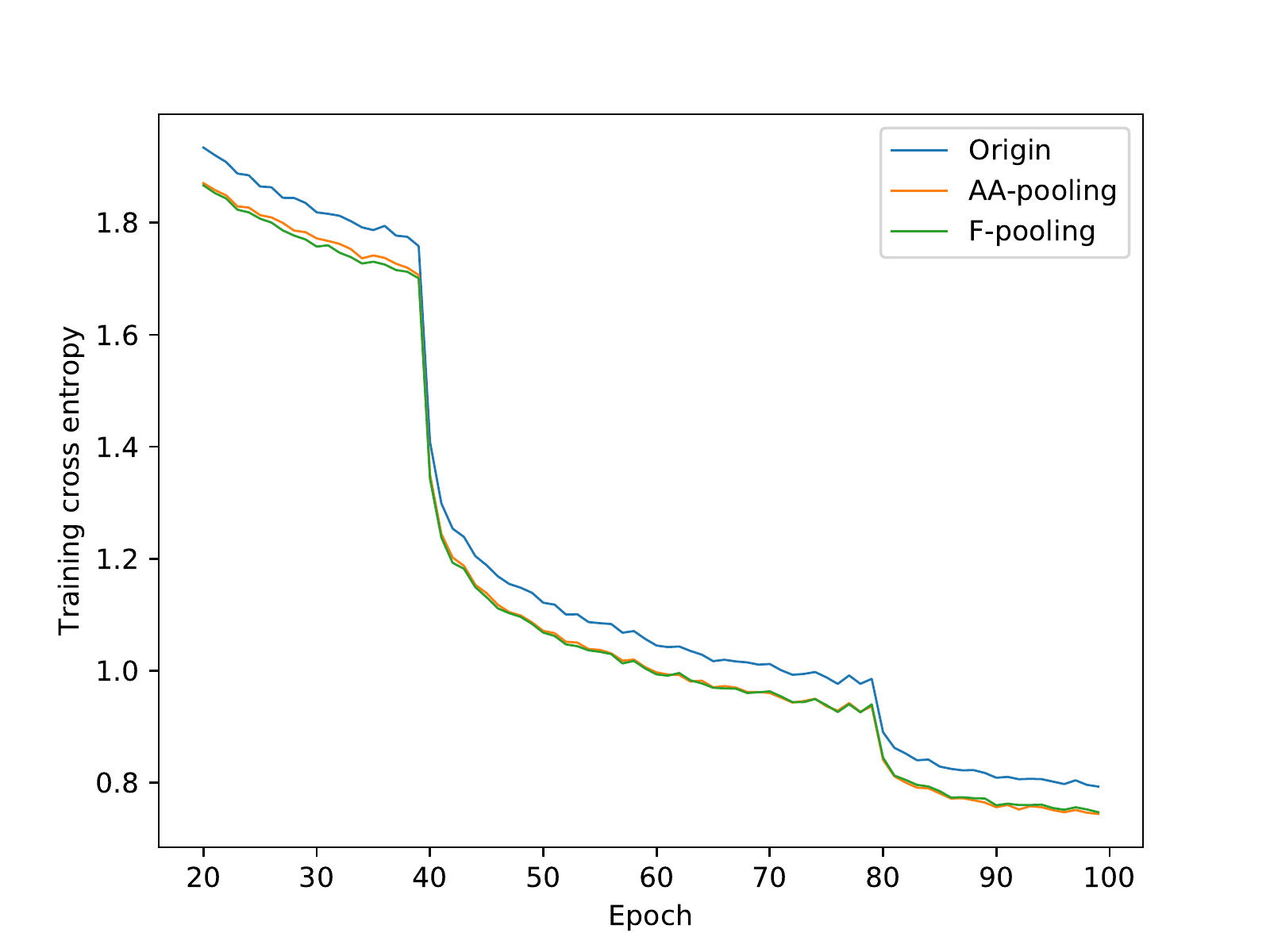}
	}
	\subfigure[MobileNet-test]{
		\includegraphics[width=0.4\textwidth]{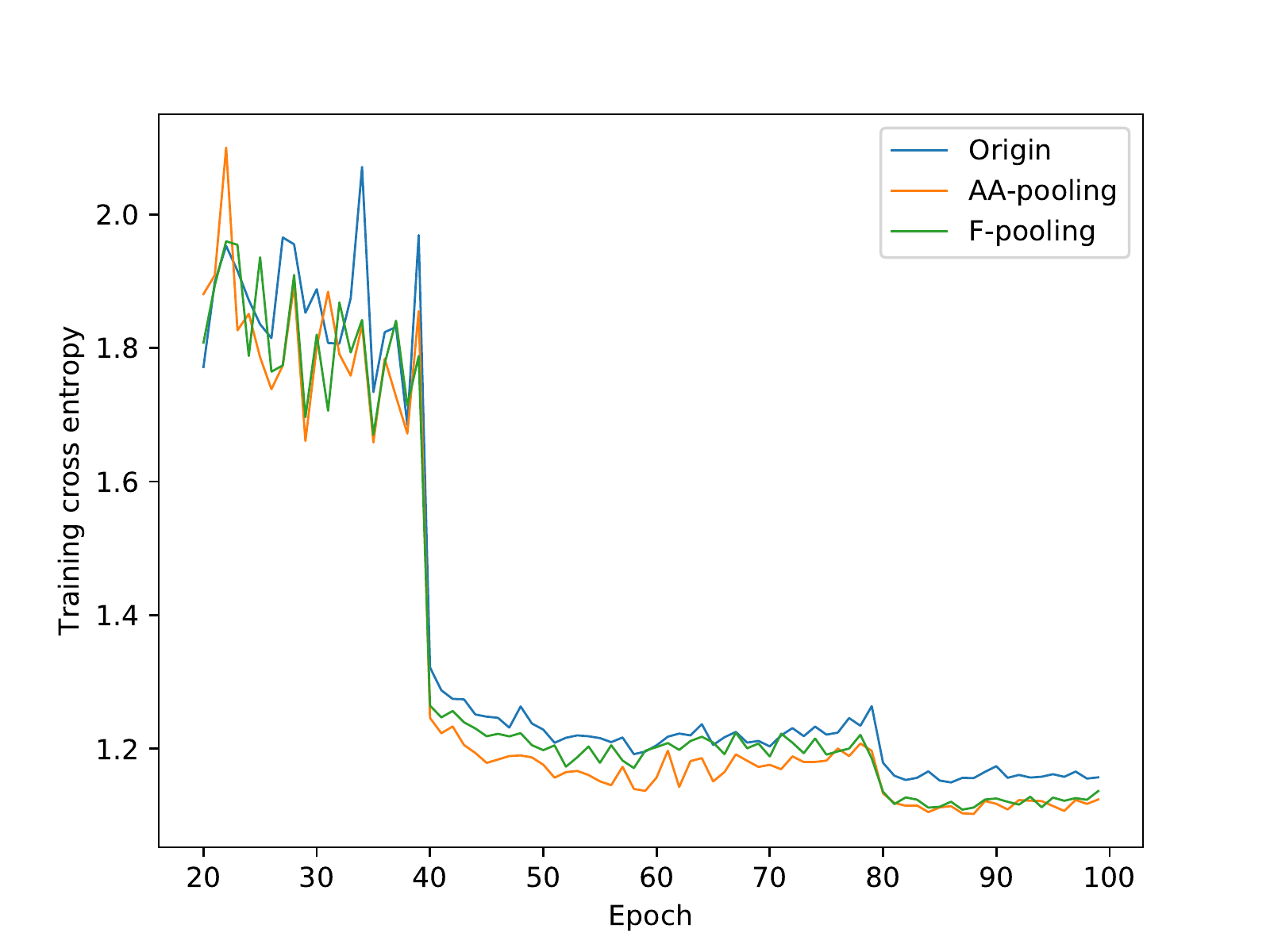}
	}
	\caption{Loss curves on sub-ImageNet. The left and right columns show the cross-entropy on training and test sets, respectively. The first, second and third rows show the cross-entropy on ResNet, DenseNet and MobileNet, respectively. Best viewed on screen. }
	\label{fig:loss}
\end{figure*}

\subsection{Ablation study}

\subsubsection{Circular padding} As mentioned in Section \ref{sec:practical}, circular padding may improve the consistency of F-pooling. In this section, we evaluate the effects of padding modes on the consistency of CNNs. We set the padding mode of all convolutions to circular padding, instead of zero padding used in Tables \ref{tab:cifar} and \ref{tab:imagenet}. To further evaluate the robustness to the shifts, we use the standard deviation (std) of probabilities of the corresponding label. We show
std and consistency on CIFAR-100 in Table \ref{tab:imagenet_circular}. We provide std and consistency on sub-ImageNet in Table \ref{tab:cifar_circular}. Based on those results, it can be concluded that
\begin{itemize}
	\item Compared with zero padding, circular padding improves the consistency of CNNs for all pooling methods.
	\item F-pooling gets more benefits from circular padding. With circular padding, F-pooling performs significantly better than AA-pooling in terms of consistency.
\end{itemize}

\subsubsection{Reconstruction}
F-pooling is optimal from the perspective of reconstruction. This is because F-pooling keeps all frequency components lower than Nyquist frequency. One may ask if this reconstruction optimality is necessary for classifications and what would happen if we retain fewer frequency components. Thus, we choose the rate of the retained frequency components from $\{50.0\%, 37.5\%, 25.0\%\}$. We evaluate the accuracy and consistency on CIFAR-100 when we vary the rates of the frequency components. As shown in Table \ref{tab:rec}, filtering more frequencies reduces accuracy but has less effect on the consistency. The optimal reconstruction of F-pooling is helpful for better classification.

\section{Conclusions}
In this paper, we have proposed F-pooling for CNNs. F-pooling reduces the dimension of signals in the frequency domain. We have defined shift-equivalence of functions which contain downsamplings by introducing an upsampling. Under this definition, we have proved that F-pooling is the optimal anti-aliasing downsampling with shift-equivalence. We have integrated F-pooling into modern CNNs. We have verified that F-pooling remarkably increases accuracy and robustness with respect to the shifts of modern CNNs. We believe that F-pooling plays an important role in the applications that shift-equivalence is required, such as object detection and semantic segmentation.

%



\newpage
{\small
	\bibliographystyle{ieee}
	\bibliography{fpooling}
}
\end{document}